\documentclass{amsart}
\input{tex/shared}

\begin{document}

    \author{Vasileios Maroulas}
    \address{Department of Mathematics - University of Tennessee, Knoxville, TN 37996}
    \email[Corresponding author]{vmaroula@utk.edu}
    \thanks{This work has been partially supported by the ARO Grant 
        \# W911NF-17-1-0313, and the NSF DMS-1821241.}

    \author{Cassie Putman Micucci}
    \address{Department of Mathematics - University of Tennessee, Knoxville, TN 37996}
    \email{cputman@vols.utk.edu}

    \author{Adam Spannaus}
    \address{Department of Mathematics - University of Tennessee, Knoxville, TN 37996}
    \email{aspannaus@utk.edu}
    
    
    \title{A Stable Cardinality Distance for Topological Classification}
    
    \begin{abstract}
        This work
incorporates topological features via persistence 
diagrams to classify point cloud data arising from materials science.
Persistence diagrams are multisets
summarizing the connectedness and holes
of given data. A new distance on the space of persistence diagrams generates relevant input 
features for a classification algorithm for materials science data. This distance measures
    the similarity of persistence diagrams using
    the cost of matching points and a regularization term corresponding to cardinality differences between diagrams. Establishing stability properties of this distance provides theoretical 
justification for the use of the distance in comparisons of such 
diagrams. 
The classification scheme succeeds in determining the crystal structure of materials
    on noisy and sparse data retrieved from synthetic atom probe tomography experiments.
    \end{abstract}
    
    \maketitle
    
    
    \section{Introduction}
    A crucial first step in understanding properties of a crystalline material is determining its
crystal structure. For highly disordered metallic alloys, such as high-entropy alloys (HEAs),
atom probe tomography (APT) gives a snapshot of the local atomic environment.
APT has two main drawbacks: experimental noise and missing data. Approximately 65\%
of the atoms in a sample are not registered in a typical experiment, and
those atoms that
are captured have their spatial coordinates corrupted by experimental noise.
As noted by \cite{kelly2013atomic} and \cite{miller2012future}, APT has a spatial 
    resolution approximately the length of the unit cell
    we consider, as seen in~\cref{fig:cells}. Hence the process is unable to see the finer details of a material, making the
    determination of a lattice structure a challenging problem. Existing algorithms for detecting the crystal 
    structure 
    \cite{chisholm2004new,hicks2018aflow,honeycutt1987molecular,larsen2016robust,moody2011lattice,togo2018spglib} are not able to
    establish the crystal lattice of an APT dataset,
    as they rely on symmetry arguments. Consequently, the field of atom probe crystallography, i.e., 
    determining the crystal
    structure from APT data, has emerged in recent years~\cite{gault2012atom}
    and~\cite{moody2011lattice}. These algorithms rely
    on knowing the global lattice structure \emph{a priori} and aim to
    determine local small-scale structures within a larger sample. For some
    materials this information is readily known, for others, such
    as HEAs, the global structure is unknown and must be inferred.
A recent work by~\cite{ziletti2018insightful} proposes a machine-learning  approach to classifying crystal structures
    of a noisy and sparse materials dataset, without 
    knowing the global structure \emph{a priori}.
The authors employ a convolutional
neural network for classifying the crystal structure by looking at
a diffraction image, a computer-generated diffraction pattern.
The authors suggest their method could be used to determine
the crystal structure of APT data or other noisy and sparse data from materials science. However, the synthetic data
considered in~\cite{ziletti2018insightful} is not a realistic
representation of experimental APT data, where about 65\% of the data
is missing~\cite{santodonato2015deviation} and is corrupted by more observational
noise~\cite{miller2012future}. Most importantly, their synthetic data is either sparse or noisy, not a combination of both. We consider a combination of noise and sparsity, such as is the case in real APT data.

\begin{figure}
    \begin{subfigure}[t]{0.48\linewidth}
        \centering
        \includegraphics[width=0.45\textwidth]{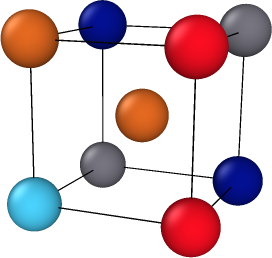}
        \caption{BCC cell}
        \label{fig:BCC}
    \end{subfigure}%
    \hfill
    \begin{subfigure}[t]{0.48\linewidth}
        \centering
        \includegraphics[width=0.45\textwidth]{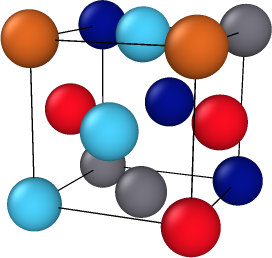}
        \caption{FCC cell}
        \label{fig:FCC}
    \end{subfigure}%
    \caption{Example of body-centered cubic, (BCC), (\textsc{a}) and face-centered
            cubic, (FCC), (\textsc{b}) unit cells
            without additive noise or sparsity. Notice there is 
            an essential topological difference between the two structures:
            The body-centered cubic structure has one atom at its center, whereas the face-centered cubic
            is hollow in its center, but has one atom in the middle of each of its faces.}\label{fig:cells}
\end{figure}

In this work, we provide a machine learning approach to classify the crystal 
structure of a noisy and 
sparse materials dataset.
Specifically, we consider materials that are either body-centered cubic (BCC) 
    or face-centered cubic (FCC), as these lattice structures are the 
    essential building blocks of HEAs~\cite{zhang2014microstructures} and have
    fundamental differences that set them apart in the case of noise-free, complete materials data.
The BCC structure has
a single atom in the center of the cube, while the FCC has a void in its
center but has atoms on the center of the cubes' faces, see~\cref{fig:cells}. These two crystal structures are distinct when viewed through the lens of Topological Data Analysis (TDA). Differentiating between
    the holes and connectedness of these two lattice structures allows us to create an
    accurate classification rule.
This fundamental distinction between BCC and FCC point clouds
    is captured well by topological methods
    and explains the high degree of accuracy
    in the classification scheme presented herein.
TDA
provides input features for machine 
learning algorithms, as well as a useful toolbox for classification. Several authors have
used TDA on real-world problems,
see~\cite{Carlsson:2004:PBS:1057432.1057449,Edelsbrunner2002,
    Marchese2,MMM,2018arXiv180302739M,
    maroulas2019bayesian, wasserman2018topological,Zomorodian2005} and the references therein. 
Persistent homology, which measures changes in topological features over different scales,
is the main framework considered by these authors.

Persistent homology is applicable to classification problems as it studies 
and differentiates holes within data as viewed in different dimensions, e.g., the space enclosed by a loop is a one-dimensional hole.
Overall, persistent homology provides a summary of the connectedness and holes (empty space in atomic cells) of data, which indirectly gives information about the shape of the data as well.
Indeed, persistent homology records when different homological features emerge and vanish in the data.
This analysis quantifies the significance of a homological feature and 
provides a tool to contend with noisy data.
The appearance and disappearance of each homological feature is calculated and recorded in a persistence
diagram. Persistence diagrams yield topological summaries
of the persistent homology of a dataset
and are rich sources of detail about underlying topological features. 
The diagrams could be used in distance-based
classifiers~\cite{carriere2017sliced,Marchese2018} or vectorized and 
input into standard classification algorithms, such as support vector 
machines~\cite{adams2017persistence, bubenik2015statistical}.

Distances on the space of  persistence diagrams yield a means of comparison between diagrams.
The Wasserstein and bottleneck distances compute the cost of the optimal matching between the
points in each persistence diagram, while allowing matching to additional points on the diagonal
to allow for cardinality differences and to prove stability properties as in \cite{Cohen-Steiner2007}.
Motivated by \cite{Marchese2018}, we consider here the $d_p^c$ distance, a distance on the space of persistence diagrams.
This distance
employs the cardinality of the persistence diagrams, as
well as distances between points in the diagrams.
It calculates the cost of an optimal matching between the persistence diagrams without any points added to the diagonal. A regularization term then considers the cardinality differences between persistence diagrams.

The stability of the $d_p^c$ distance is also verified in this paper.
This property guarantees that when the distances between point clouds go to zero, the
distances between the associated persistence diagrams
go to zero as well.
Another formulation of this stability is given in 
\cite{Chazal2014}; using a related approach, we show continuity of the mapping of point cloud to persistence diagram under the $d_p^c$ distance.
This analysis provides insight into how the cardinality of the diagrams changes 
with the size of the input point clouds. 
Additionally, using statistics on the diagram's cardinality generates corresponding 
prediction intervals, which give probabilistic bounds on the $d_p^c$ distances between persistence diagrams. The idea is that point clouds generated from the same process 
have small variability with respect to cardinality of the persistence diagrams.

The contributions 
    of this work is: \begin{enumerate}[topsep=0pt, partopsep=0pt]
        \item The stability of the $d_p^c$ distance in a continuous fashion.
        \item Theoretical and statistical bounds on the number of 1-dim holes represented in a persistence diagram based on the cardinality of the underlying point cloud.
        \item A $\dpc$ distance based classification algorithm for the crystal structure of high entropy alloys using synthetic atom probe tomography experiments.
\end{enumerate}

The work is organized as follows.  Relevant definitions and concepts necessary for 
persistent homology are presented in Section \ref{sect:ph}. Stability results of the
$d_p^c$ distance are in Section \ref{sect:stability}, as well as prediction interval bounds.
Section \ref{sect:classification Methodology} demonstrates a classification scheme for materials science data retrieved from synthetic APT experiments.
We conclude and provide future directions in Section \ref{sect:conclusions}.
    %
    %
    %
    %
    \section{Persistent Homology Background}\label{sect:ph}

This section succinctly explains the construction of persistence diagrams, which are topological summaries of the underlying space.
The Vietoris-Rips complex
provides the necessary computational link between the point cloud, a subset of $\mathbb{R}^d$ under the Euclidean distance, and its persistence diagram.
Below we give a brief summary of the necessary background. 
For a detailed treatment, see~\cite{Ed2010}.

\begin{definition}
        A $\nu$-simplex is the convex hull of an affinely independent point set of size $\nu+1$.
        \label{def:simplex}
\end{definition}

\begin{definition}
For a set of points $\mathcal{P}$, an abstract simplicial complex $\sigma$ is a collection of finite subsets of $\mathcal{P}$ such that for
        every set $A$ in $\sigma$ and every nonempty set $B \subset A$, we have that $B$ is in $\sigma$.
        The elements of $\sigma$ are called abstract simplices and are the combinatorial analogues of the geometric simplices in Def. \ref{def:simplex}.
\end{definition}

\begin{definition}
    For a given threshold $\epsilon$, the Vietoris-Rips complex is a simplicial complex formed from a set such that 
    corresponding to each subset of $\nu$ points of the set, an $\nu$-simplex is 
    included in the Vietoris-Rips complex each time the subsets have pairwise 
    distances at most $\epsilon$.
\end{definition}

The Vietoris-Rips complex can be visualized by placing a ball of radius $\epsilon/2$ at each point in the set and then adding a $\nu$-simplex at the points corresponding to the intersection of $\nu$ balls.
See~\cref{fig:VR_plot} for an illustration. For the Vietoris-Rips complex
corresponding to $\epsilon$, denoted by $VR_{\epsilon}$, it is clear that $VR_{\epsilon}
\subset VR_{\epsilon'}$ for $\epsilon < \epsilon'$. Thus we need only examine specific
$\epsilon$ values corresponding to the emergence and disappearance of homological features.
These $\epsilon$ values are recorded as ordered pairs $(b,d)$ in a persistence diagram, where
$b$ denotes the birth of a feature and $d$ its death.

As can be seen in~\cref{fig:VR_plot}, a 0-dim homological 
feature is a connected component of a simplex, a 1-dim homological
feature is a hole, such as those created by a loop or the circle $S^1$,
and a 2-dim homological feature describes voids, e.g., the inside of a sphere; see \cite{wasserman2018topological} for details.
Higher dimensional data analogously yields higher dimensional holes. 

    \begin{remark}
        Persistence diagrams can also be computed using a pertinent function $g$ from a topological space to $\mathbb{R}$.
        Such a function can act as an approximation to a point cloud; typical functions used are kernel density estimators as in \cite{fasy2014} and the distance to measure function as in \cite{Chazal2011}. 
        Homological features are born and die within the sublevel sets $g^{-1} (-\infty, t]$ as $t$ increases.
        These birth and death times create another persistence diagram,
        see~\cref{fig:VR6}.
        \label{remark:sublevel}
    \end{remark}

\begin{figure}
    \centering
    \begin{subfigure}[t]{0.38\textwidth}
        \centering
        \includegraphics[width=\textwidth]{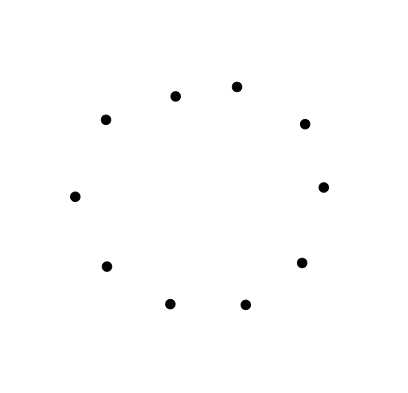}
        \caption{}
        \label{fig:VR1}
    \end{subfigure}
    \hfill%
    \begin{subfigure}[t]{0.38\textwidth}
        \centering
        \includegraphics[width=\textwidth]{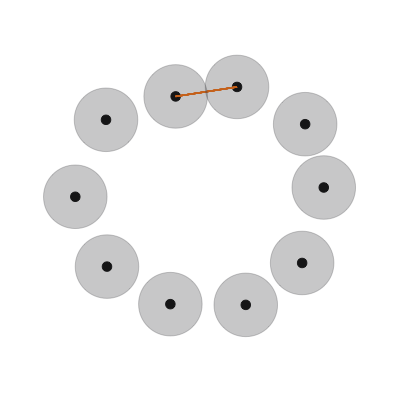}
        \caption{}
        \label{fig:VR2}
    \end{subfigure}
    \begin{subfigure}[t]{0.38\textwidth}
        \centering
        \includegraphics[width=\textwidth]{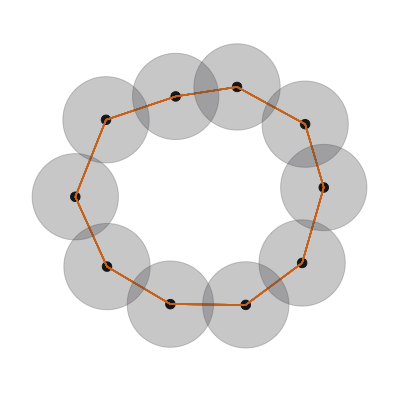}
        \caption{}
        \label{fig:VR3}
    \end{subfigure}
    \hfill%
    \begin{subfigure}[t]{0.38\textwidth}
        \centering
        \includegraphics[width=\textwidth]{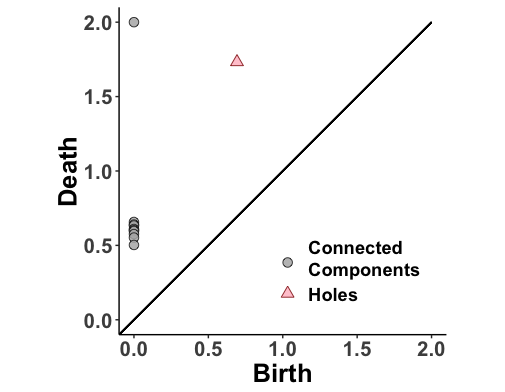}
        \caption{}
        \label{fig:VR4}
    \end{subfigure}
    \begin{subfigure}[t]{0.38\textwidth}
        \centering
        \includegraphics[width=\textwidth]{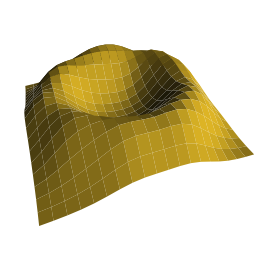}
        \caption{}
        \label{fig:VR5}
    \end{subfigure}
    \hfill%
    \begin{subfigure}[t]{0.38\textwidth}
        \centering
        \includegraphics[width=\textwidth]{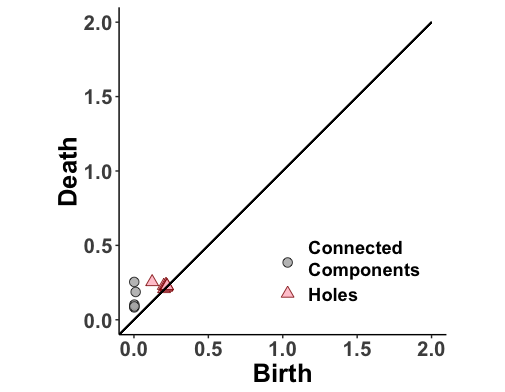}
        \caption{}
        \label{fig:VR6}
    \end{subfigure}
    \caption{Begin with a point cloud (\textsc{a}). After increasing the radius of the balls around the points, a 1-simplex (line segment) forms in the corresponding Vietoris-Rips complex, (\textsc{b}). 
        Eventually, more 1-simplices are added and a 1-dim hole forms (\textsc{c}).
        In (\textsc{d}), the persistence diagram tracks all the birth and death times, with respect to the radius $\epsilon/2$, of the homological features for each dimension. 
        Using the same points as in (\textsc{a}), the kernel density estimator function for this point cloud is plotted
        in (\textsc{e}). 
        A corresponding persistence diagram is created using sublevel sets in (\textsc{f}). 
        Note the difference between the persistence diagrams in (\textsc{d}) and 
        (\textsc{f}). The persistence diagram created in (\textsc{f}) has noisy 1-dim features that are not present in the persistence diagram created directly from the data points.
        \label{fig:VR_plot}}
\end{figure}

To calculate the similarity between diagrams for classification problems, a distance on the space of
persistence diagrams is needed.
A typical distance is the Wasserstein distance.

\begin{definition}
    The $p$-Wasserstein distance between two persistence diagrams $X$ and $Y$ is given by $   W_p (X,Y) = \left( \inf_{\eta:X \to Y} \sum_{x \in X} \| x - \eta (x) \|_{\infty}^p  \right)^{\frac{1}{p}}$,
    where the infimum is taken over all bijections $\eta$, and the points of the diagonal are added with infinite multiplicity to each diagram.
    If $p \to \infty$, then $W_{\infty} (X,Y) \; =  \inf_{\eta:X \to Y} \sup_{x \in X} \| x - \eta (x) \|_{\infty}  $ is the bottleneck distance between diagrams $X$ and $Y$.
    
\end{definition}

The Wasserstein distance yields the penalty of matched points under the optimal bijection.
    Points can be matched to the diagonal of each persistence diagram, which is assumed to have infinitely many points with infinite multiplicity; this 
    ensures that a bijection between $X$ and $Y$ actually exists, since $X$ and $Y$ may not have the same cardinality.
    In other words, the Wasserstein distance gives no explicit penalty for differences in cardinality between two diagrams.
    Instead, the Wasserstein distance penalizes unmatched points by using their distance to the diagonal.
    However, cardinality differences may play a key role in machine learning problems, and
    to that end,~\cite{Marchese2018} proposed the $d_p^c$ distance given below.

\begin{definition}
    Let $X$ and $Y$ be two persistence diagrams with cardinalities $n$ and $m$ 
    respectively such that $n \leq m$ and denoted $X = \{ x_1, \ldots, x_n \}$,
    $Y = \{ y_1, \ldots ,y_m\}$.
    Let $c >0$ and $1 \leq p < \infty$ be fixed 
    parameters. The $d_p^c$ distance between two persistence diagrams $X$ and $Y$ is
    \begin{equation}
    d_p^c(X,Y) = \left( \frac{1}{m} \left( \min_{\pi \in \Pi_m} \sum_{\ell=1}^n \min(c,\|x_{\ell}-y_{\pi (\ell)} \|_{\infty})^p + c^p |m-n| \right) \right)^{\frac{1}{p}},\label{eqn:dpc}
    \end{equation}
    where $\Pi_m$ is the set of permutations of $(1, \dots, m)$. If $m < n$,
    define $d_p^c (X,Y) := d_p^c (Y,X)$.
    
    \label{def:dpc}
\end{definition}

\begin{figure}
    \centering
    \begin{subfigure}[t]{0.44\textwidth}
        \centering
        \includegraphics[width=\textwidth]{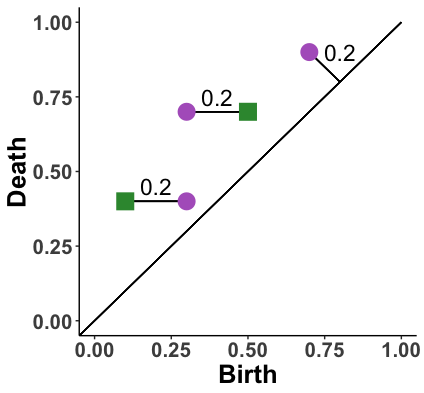}
        \caption{}\label{fig:wass_dist}
    \end{subfigure}
    \hfill
    \begin{subfigure}[t]{0.44\textwidth}
        \centering
        \includegraphics[width=\textwidth]{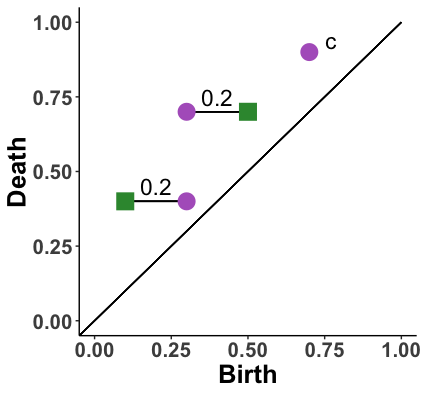}
        \caption{}\label{fig:dpc_dist}
    \end{subfigure}
    
    \caption{Consider two persistence diagrams, one given by the green squares and another by the purple circles. (\textsc{a}) The Wasserstein distance 
        imposes a cost of 0.2 to the extra purple point (the $\ell^{\infty}$-distance to the diagonal). (\textsc{b}) The $d_p^c$ distance imposes a penalty $c$ on the point instead. 
        \label{fig:dpc_wass}}
\end{figure}

\begin{remark} 
    Note that this distance can be applied to arbitrary 
        point clouds with finite cardinality as well.
    As shown in~\cite{Marchese2018}, a smaller $c$ in Eq.~\eqref{eqn:dpc} accounts for local geometric differences, while a larger $c$ focuses on global geometry.
    It is precisely by considering 
    differences in cardinality that the $d_p^c$ distance can distinguish between features of the point cloud that other distances may miss.
        Also 
        in Eq.~\eqref{eqn:dpc}, if $X$  is fixed and $m \to \infty$ , then $d_p^c (X,Y) \to c$.
\end{remark}


\section{Stability Properties for $d_p^c$ distance}\label{sect:stability}

The stability  of the $d_p^c$ distance is proved in this section.
Stability of the distance under investigation means that small perturbations in the underlying space result
in small perturbations of the generated persistence diagrams.  
Adopting the approach of estimating a point cloud via a pertinent function, e.g., a kernel density estimator
as in~\cite{fasy2014}, persistence diagrams may be constructed using sublevel sets as 
in~\cref{fig:VR6} and Remark \ref{remark:sublevel}.
Their differences can be computed using the Wasserstein and bottleneck distances.
Using this functional representation, stability of the Wasserstein and bottleneck
distances has been shown in~\cite{Cohen-Steiner2010} and~\cite{Cohen-Steiner2007}
respectively, by verifying Lipschitz (and respectively H\"{o}lder) 
continuity of the mapping from the underlying function
of the data to its persistence diagram in the bottleneck and Wasserstein distances.
Considering discrete point clouds whose distances shrink to zero, 
Theorem~\ref{thm:stability} shows that the
distance between persistence diagrams goes to zero as well.

\begin{theorem}[Stability Theorem]\label{thm:stability}
    Consider $c >0$ and $\,1 \leq p < \infty $. Let $A$ be a finite nonempty point cloud in $\mathbb{R}^d$. 
    Suppose that $\{A_i\}_{i \in \mathbb{N}}$ is a sequence of finite nonempty point clouds such that $d_p^c (A,A_i) \to 0$ as $i \to \infty$. 
    Let $X^{k} and \; X_i^{k}$ be the $k$-dim persistence diagrams created from the Vietoris-Rips complex for $A$ and $A_i$ respectively.  
    Then $d_p^c (X^{k},X_i^{k}) \to 0 $ as $i \to \infty$. 
\end{theorem}

Note that Theorem~\ref{thm:stability} does not depend on a function created from the points such as a kernel density estimator as in~\cite{fasy2014}, but simply on the points themselves and the Vietoris-Rips complex generated from these points. 
In fact, Theorem \ref{thm:stability} shows that the mapping from a point cloud to the persistence diagram of its Vietoris-Rips complex is continuous under the $d_p^c$ distance.
This continuous-type stability result is weaker than Lipschitz stability.
In order to prove Theorem~\ref{thm:stability}, we first show that if the $d_p^c$ distance between the underlying point clouds goes to 0, then eventually the size of the point clouds must be the same. 

\begin{lem}\label{lemma:lemma1}
    Let $A$ and $A_i$ be as in Theorem \ref{thm:stability} such that $d_p^c (A,A_i) \to 0$ as $i \to \infty$. Then $A_i$ and A have the same number of points for $i \geq N_0$ for some $N_0 \in \mathbb{N}$.
\end{lem}

\begin{proof} 
    Denote by |A| the number of points in the point cloud A.
    Suppose that $|A_i| \neq |A|$
    infinitely often. 
    Since $d_p^c (A,A_i) \to 0$, for every $\epsilon > 0$, there is
    an $N \in \mathbb{N}$ such that $i \geq N$ implies that $d_p^c
    (A,A_i) < \epsilon$.
    Let $\epsilon = \frac{c}{|A|+1}$, noting that $|A|$ is fixed. 
    By assumption $|A_i| < |A|$, $|A_i| > |A|$, or both, infinitely
    often. If $|A| < |A_i|$, then by Def. \ref{def:dpc}
    \begin{equation}
    d_p^c (A,A_i) \geq  \left(c^p \frac{|A_i|-|A|}{|A_i|}  \right)^\frac{1}{p} 
    \geq c \frac{|A_i|-|A|}{|A_i|}.
    \label{equation:dpclemma}
    \end{equation}
    
    The function $h: \mathbb{N} \to \mathbb{R}$ given by $h(z) = \frac{z - |A|}{z}$ is strictly increasing. 
    Whenever $|A| < |A_i|$, we have $|A_i| \geq |A| + 1$.
    The restriction of $h$ to $\{|A|+1, |A| + 2, |A| + 3, \ldots \} $ achieves its minimum at $|A| + 1$.
    This shows that the RHS of Eq.~\eqref{equation:dpclemma} is greater than or equal to $ \frac{c}{|A|+1},$
    whenever $|A| < |A_i|$, which by assumption happens infinitely often. This contradicts $d_p^c (A,A_i) < \epsilon$  for all $i \geq N$. The  case where $|A| > |A_i|$  follows similarly. $\qed$
\end{proof}

\begin{lem}\label{lemma:lemma3}
   Let $A$ and $A_i$ be as in Theorem \ref{thm:stability}.
    Suppose the points of each point cloud $A_i$ are ordered so that 
    $A_i = \{ a_{\pi_i (1)}, a_{\pi_i (2)},
    \ldots, a_{\pi_i(|A|)} \}$, where $\pi_i$ is the permutation used to calculate the $d_p^c$ distance between $A_i$ and $A$ as in Eq.~\eqref{eqn:dpc}.
    Let $D_A$ and $D_{A_i}$ be the distance
    matrices for the points of $A$ and $A_i$ respectively, 
    i.e., the $k l$-th entry of $D_A$ is $\|a_{k}-a_l \|_d $. Then,
    
    \begin{enumerate}[label=(\roman*)]
        \item $ \|D_A - D_{A_i} \|_{\infty} \to 0 $ as $i \to \infty$, and
        \item for some $N_1 \in \mathbb{N}$, the order of the entries of the upper triangular portion of $D_A$ and $D_{A_i}$ is the same for $i \geq N_1$, up to permutation when either $D_A$ or $D_{A_i}$ have duplicate entries.
    \end{enumerate}
\end{lem}

\begin{proof}
    (i) Let $A = \{a_1, \ldots a_k \}$, $A_i = \{a^i_1, \ldots a^i_k \} $, and ${\lambda}^i_{\alpha} = \|a_{\alpha}-a^i_{\pi_i(\alpha)} \|_d $ for the permutation $\pi_i$ in the $d_p^c$ distance between $A_i$ and $A$.
    Suppose that $d_p^c (A,A_i) \to 0$. 
    Note that since $c$ is fixed, then by Lemma \ref{lemma:lemma1}, there is some $N_c$ such that eventually 
    $d_p^c (A_i,A) = \left( \frac{1}{|A|} \min_{\pi_i \in \prod_{|A|}} \sum_{\ell = 1}^{|A|} \| a_{\ell} - a_{\pi_i(\ell)} \|_d^p \right)^{\frac{1}{p}} $ 
    for $i \geq N_c$.
    By assumption $d_p^c (A,A_i) \to 0$, which shows that  $|A|^{-\frac{1}{p}} \|  \lambda \|_p \to 0$ as $i \to \infty$.
    Thus  $ \| \lambda^i \|_p \to 0$ as $i \to \infty$.
    
    Now, let $E = D_A - D_{A_i}$. 
    \begin{align*}
    \| E \|_{\infty} &= \max_{k,l}  \big| \|a_k - a_l\|_d - \|a_k^i - a_l^i \|_d \big| \\
    &= \max_{k,l}  \big| \|a_k - a_l\|_d + \|a_l - a_k^i \|_d - \|a_l - a_k^i \|_d  - \|a_k^i - a_l^i \|_d \big| \\ &\leq \big| \|a_k - a_l\|_d - \|a_l - a_k^i \|_d \big| + \big| \|a_k^i - a_l^i \|_d - \|a_l - a_k^i \|_d  \big| \\ &\leq \|a_k - a_k^i \|_d + \|a_l - a_l^i \|_d \numberthis\label{equation:sum}
    \end{align*}
    The last term in Eq.~\eqref{equation:sum} goes to 0 as $i \to \infty$, proving (i).\\
    (ii) Suppose that the $m$ distinct upper triangular entries of $D_A$ are ordered from smallest to largest, say $d_1^A < d_2^A < \cdots d_m^A$, where $m \leq {|A|(|A|-1)/2}$.  For $\eta \in \{1,\ldots ,m+1 \}$ let $h_{\eta} \subset [0,\infty)$ be a sequence such that $h_1 < d_1^A < h_2 < d_2^A < \cdots < h_{m} < d_m^A < h_{m+1}  $. 
    Let  $\|D_A - D_{A_i} \|_{\infty} < \frac{h}{2}$, where $h = \min_{\eta_1,\eta_2 \in \{1,\ldots ,m+1 \}} \{ |h_{\eta_1} - h_{\eta_2} | \}$.
    We show that there exists a sequence $g_{\eta}$ such that $|h_{\eta} - g_{\eta}|<2h $ for each $\eta \in \{1,\ldots ,m+1 \}$ and $h_{\eta} < d_j^A < h_{\eta +1}  $ implies $g_{\eta} < d_j^{A_i} \leq g_{\eta +1}  $.
    Let $h_{\eta} < d_j^A  < h_{\eta + 1}$, and suppose that it is not the case that  $ h_{\eta} < d_j^{A_i} \leq  h_{\eta +1} $. 
    Since $\|D_A - D_{A_i} \|_{\infty} < \frac{h}{2}$, then either $d_j^{A_i} \in (h_{\eta - 1}, h_{\eta}]$ or $ d_j^{A_i} \in (h_{\eta+1}, h_{\eta +2}]$.
    If the first case is true, then take $g_{\eta} = d_j^A - \frac{h}{2}$. 
    If the second, then take $g_{\eta} = d_j^A + \frac{h}{2}$. 
    This proves the existence of the sequence.
    Now proceeding by contradiction, if the lemma does not hold for some entries $d^A_j \in D_A$ and $d^{A_i}_j \in D_{A_i}$, then take 
    $\|D_A - D_{A_i} \|_{\infty} < \frac{1}{2} |d^A_j - d^{A_i}_j |  $. $\qed$
\end{proof}

\begin{proof}[Proof of Theorem \ref{thm:stability}]
    By Lemma~\ref{lemma:lemma1}, take $|A_i| = |A|$ without loss of generality. 
    By Lemma~\ref{lemma:lemma3} (i), $ \|D_A - D_{A_i} \|_{\infty} \to 0 $ as $i \to \infty$. 
    If the Vietoris-Rips complex were computed at every threshold value in $[0, \infty)$, then the birth and
    death times of all features of all dimensions would be distances between points in the underlying point
    cloud (including the birth time of 0 in the 0-dim diagram). 
    Since the order of the entries of $D_A$ and $D_{A_i}$ may be taken to be the same 
    from Lemma~\ref{lemma:lemma3} (ii), the same number of simplices are formed 
    in the complexes for $A$ and $A_i$ for each dimension of simplex.
    Also, the labels of the simplices according to the points of $A$ and $A_i$ are given from the permutation
    $\pi_i$ in Lemma~\ref{lemma:lemma3} (i).
    
    Now, for 0-dim it is clear that for the cardinalities of the persistence diagrams, $|X^{0}| = |X_i^{0}|$
    since for the sizes of their associated point clouds, $|A_i| = |A|$.
    For a higher dimensional feature ($k \geq 1$) to appear in the complex, we must have that a certain
    number of the distances are less than or equal to the threshold $\epsilon$ and a certain number of the
    distances are greater than $\epsilon$. 
    Lemma \ref{lemma:lemma3} (ii) shows that although the thresholds where the features are created may be
    different, the same number of features are formed in the Vietoris-Rips complexes of $A$ and $A_i$, and
    these  features are formed in the same order and with the points that correspond under $\pi_i$.
    
    If $X^{k} = \{ x_{1}, x_2, \dots, x_{|X^{k}|} \} $ and 
    $X_i^{k} = \{ x_{1}, x_2, \dots, x_{|X_i^{k}|} \} $, then
    we have that $|X^{k}| = |X_i^{k}|$ and that $d_p^c 
    (X^{k}, X_i^{k}) < 2h$. Thus $d_p^c 
    ( X^{k}, X_i^{k}) \to 0 $ as $i \to \infty$. $\qed$
\end{proof} 

To provide a practical way to control $c$ in computing the $d_p^c$ distance 
of Eq.~\eqref{eqn:dpc} and
consequently compute the possible fluctuations of the $d_p^c$ distance, a 
probabilistic upper bound, which relies on least squares, is provided. 
Specifically, the following analysis gives predictions on the number of 1-dim holes
represented in the persistence diagram, which we denote by $b_1$. The 
parameter $b_1$ relies on the number of connected
components (or equivalently the number of points in the point cloud) represented 
in the persistence diagram, denoted by $b_0$.

    \begin{definition}[\cite{pfender2004kissing}]
        The kissing number in $\mathbb{R}^d$ is the maximum number of nonoverlapping unit spheres that can be arranged so that each touches another common central unit sphere. 
    \end{definition}
    
    \begin{lem}[\cite{Goff2011}]
        For a finite point cloud with no more than $\rho$ points in $\mathbb{R}^d$ under the Euclidean distance,
        let $M_d (\rho)$ denote the maximum possible number of 1-dim holes in the Vietoris-Rips complex for the point cloud for a given threshold.
        Then
        \begin{equation}
        M_d (\rho) \leq (K_d-1) \rho .
        \end{equation}
        \label{lemma:goff}
    \end{lem}

\begin{proposition}
    Consider a point cloud in $\mathbb{R}^d$ with $\rho$ points and its associated persistence diagram. Let
    $B_1$ denote the possible range of the number of 1-dim holes $b_1$. Then $B_1$ is such that
    $\{0,1,\ldots, \lfloor\frac{\rho}{2}\rfloor -1 \} \subseteq B_1 \subseteq \{0,1, \ldots, 
    \frac{1}{2}(K_d-1) \rho^2 (\rho-1) \},$ i.e., the possible range of $b_1$ is expanding as the number of points, $b_0$, in the point cloud increases.
    \label{proposition:holes}
\end{proposition}

\begin{proof}
    We first show the inclusion $\{0,1,\ldots, \lfloor\frac{\rho}{2}\rfloor -1 \} \subseteq B_1$. To form a point cloud with $\rho$ points that has $b_1 = 0$, simply
    take the $\rho$ points and arrange them on a line. 
    To form a point cloud with $\rho$ points that has 
    $b_1 = \lfloor\frac{\rho}{2}\rfloor-1$, arrange the $\rho$ points in
    two rows each with $\lfloor\frac{\rho}{2}\rfloor$ points. 
    Set the spacing between adjacent points in each of the rows to be
    1 and then place the two rows directly beside each other so that
    for each point in the first row, there is exactly one point in
    the second row at a distance of 1. Adding edges appropriately
    creates $b_1 = \lfloor\frac{\rho}{2}\rfloor-1$ squares with side
    length 1. Thus, creating the Vietoris-Rips complex and
    corresponding diagram gives
    $b_1 = \lfloor\frac{\rho}{2}\rfloor-1$.
    For an illustration of the arrangement, see~\cref{fig:Holesplot2}.
    
    To form a point cloud with $\rho$ points that has $b_1 \in \{1,2,\dots \lfloor\frac{\rho}{2}\rfloor-2 \}$,
        arrange $2 (b_1 +1)$ points in two rows as in~\cref{fig:Holesplot2}. 
        Arrange the other $\rho-2(b_1 +1) $ points in a line with the minimum distance
        from any points in the line to points of the two rows such that it is greater than
        or equal to 1. 
        Then exactly $b_1$ holes are formed from the two rows, with no holes formed by the
        line. For an illustration, see~\cref{fig:Holesplot3}. 
    
    Next, we verify the inclusion  $B_1 \subseteq \{0,1, \ldots, \frac{1}{2}(K_d-1) \rho^2 (\rho-1) \}$.
        By Lemma~\ref{lemma:goff}, the number of 1-dim holes in the Vietoris-Rips complex for a fixed radius $\epsilon$ for the point cloud is bounded above by $(K_d-1) \rho$.
        The homology of the Vietoris-Rips complex changes at most ${\rho \choose 2}$ times as the radius $\epsilon$ increases due to the maximum of ${\rho \choose 2}$ distinct distances between points in the point cloud.
        Therefore, there can be at most $\frac{1}{2}(K_d-1) \rho^2 (\rho-1)$ 1-dim holes formed over the entire evolution of the Vietoris-Rips complex. 
        This gives the desired bound of $b_1 \leq \frac{1}{2}(K_d-1) \rho^2 (\rho-1)$. 
    $\qed$
\end{proof}

\begin{figure}
    \centering
    \begin{subfigure}[b]{0.40\textwidth}
        \centering
        \centering
        \includegraphics[width = 2.2in]{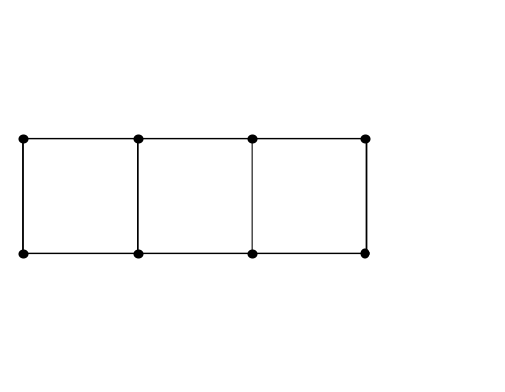}
        \vspace{-1cm}
        \caption{}
        \label{fig:Holesplot2}
        
    \end{subfigure}
    \hfill
    \begin{subfigure}[b]{0.49\textwidth}
        \centering
        \includegraphics[width = 2.2in]{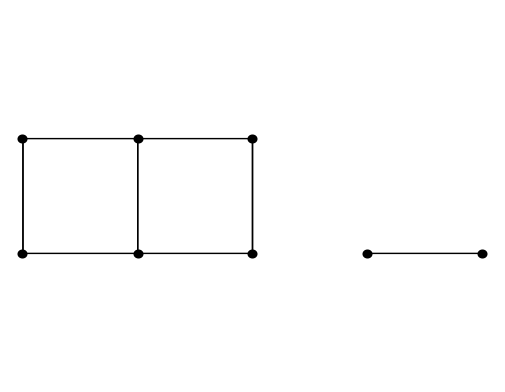}
        \vspace{-0.73cm}
        \caption{}
        \label{fig:Holesplot3}
    \end{subfigure}

    \caption{An example of 8-point arrangements to visualize the proof of 
        Proposition~\ref{proposition:holes}. (\textsc{a}) A 3-hole configuration 
        vs.~(\textsc{b}) a 2-hole configuration.}
\end{figure}

Now, let $N$ point clouds be generated from some process, and $N$ corresponding persistence diagrams be
created. For each persistence diagram $X_i^{k}, k \in \{0,1 \}, i = 1, \ldots, N$, record the 
cardinality $b_0^i$ of the 0-dim diagram and the cardinality $b_1^i$ of the 1-dim diagram.
Let $\boldsymbol{b_0} \in \mathbb{R}^{N \times 2}$ be the predictor matrix whose rows are $[1 \; b_0^i]$ and
$\boldsymbol{b_1} \in \mathbb{R}^{N}$ be the vector of responses with entries $b_1^i$.
Proposition \ref{proposition:holes}
gives that the possible range of $\boldsymbol{b_1}$ is increasing as $\boldsymbol{b_0}$ grows, 
which yields that an increase in variance as $\boldsymbol{b_0}$ grows may be 
present, i.e.,~heteroscedasticity exists.
Thus the analysis of the change in number of 1-dim holes as the size of the point cloud changes needs to
account for heteroscedasticity in order to capture the non-constant variance behavior.
Therefore to estimate the number of 1-dim holes, we use weighted least squares 
as in \cite{efron_hastie_2016}. 
If $\mathbf{W} \in \mathbb{R}^{N \times N}$ is the weight matrix $\mathbf{W} = \textrm{diag}(a_1, \ldots,
a_{N})$, then a weighted least-squares regression can be found for $\boldsymbol{b_1} =  \boldsymbol{b_0}
\boldsymbol{\gamma} + \boldsymbol{\epsilon}$, where $\epsilon_i \sim \mathcal{N} (0, \sigma_i^2)$. 
The approximation is then given by $\mathbf{b_0} \boldsymbol{\hat{\gamma}}  = \boldsymbol{b_1} $, with
$\hat{\boldsymbol{\gamma}} = (\mathbf{b_0}^T \mathbf{W} \mathbf{b_0})^{-1} \mathbf{b_0}^T \mathbf{W}
\boldsymbol{b_1}$. In turn, Proposition \ref{proposition:wls} provides bounds from 
prediction intervals using weighted least squares  for the $d_p^c$ distance.

\begin{proposition}\label{proposition:wls}
    Suppose $N$ point clouds are generated from a process, and $N$ corresponding persistence diagrams are created.
    For each persistence diagram $X_i^k, k \in \{0,1 \}$, record the cardinality of the 0-dim diagram $b_0^i$ and of the 1-dim diagram $b_1^i$.
    Let $\boldsymbol{b_0} \in \mathbb{R}^{N \times 2}$ be the predictor matrix whose rows are $[1 \; b_0^i]$ and $\boldsymbol{b_1} \in \mathbb{R}^{N}$ be the vector of responses of $b_1^i$.
    Assume the model $\boldsymbol{b_1} = \mathbf{b_0} \boldsymbol{\gamma} + \boldsymbol{\epsilon} $, where $\epsilon_i \sim \mathcal{N}(0,\sigma_i^2)$ depends on the value of the input $b_0^i$. 
    Let $X^1$ and $Y^1$ be persistence diagrams generated from the same process as $\boldsymbol{b_0}$
    with $|X^0| = \mu $. 
    Considering the $(1-\alpha) \cdot 100\%$-level prediction interval for $\boldsymbol{b_1}$, the distance $d_p^c(X^1,Y^1)$ is bounded above by
    \begin{equation*}
         \begin{split}
    \left(\min_{\pi \in \Pi_m} \sum_{\ell=1}^n \min(c,\|x_{\ell}^1-y^1_{\pi (\ell)} \|_{\infty})^p + c^p 2
     t_{1-\alpha,N-2} s \sqrt{[1 \; \mu ] (\mathbf{b_0}^T \mathbf{W}
         \mathbf{b_0})^{-1} [1 \; \mu ]^T + \mu} \right)^{\frac{1}{p}}.
 \end{split}
    \end{equation*}
    
\end{proposition}

\begin{proof}
    Prediction intervals can be constructed for the cardinality of
    a 1-dim diagram for an instance of point cloud size ${b_0}^*$ using
    standard results on weighted least squares.
    Specifically, for level $(1-\alpha) \cdot 100 \%$ a prediction
    interval for the new response $\widehat{b_1}^*$ is sought.
    To calculate this interval for a new response from the mean
    predicted response $\widehat{b_1}^* =
    \widehat{\boldsymbol{\gamma}}{b_0}^*$, 
    note that $\widehat{b_1}^*-{b_1}^*$ has the distribution
    $\frac{\widehat{b_1}^*-{b_1^*}}{\textrm{Var}(\widehat{b_1}^*-{b_1}^*)} \sim t_{N-2}.$ 
    Also, $\Var(\widehat{b_1}^*-{b_1}^*) = \Var
    (\boldsymbol{\epsilon}) [1 \; {b_0}^*] (\mathbf{b_0}^T \mathbf{W} \mathbf{b_0})^{-1} [1 \; {b_0}^*]^T + \frac{\Var
        (\boldsymbol{\epsilon})}{w*}$, where $w^* = \frac{1}{b_0^*}$, the
    weight corresponding to ${b_0}^*$. 
    Prediction intervals for $b_1^*$ are thus $\widehat{b_1}^* \pm
    t_{1-{\alpha / 2},N-2} s \sqrt{[1 \; {b_0}^*]  (\mathbf{b_0}^T 
        \mathbf{b_0})^{-1} [1 \; {b_0}^*]^T  + {b_0}^*},$ where $s^2 =
    \frac{\boldsymbol{\widehat{\epsilon}}^T \mathbf{W} \boldsymbol{\widehat{\epsilon}}}{N-2}$,
    the unbiased estimator for $\Var(\boldsymbol{\epsilon})$, using
    the residuals $\boldsymbol{\widehat{\epsilon}}$.
    Thus the cardinality difference term in the calculation of the $d_p^c$ distance as in Eq.~\eqref{eqn:dpc} is bounded above by the length of the prediction interval with $(1-\alpha) \cdot 100\%$-level confidence.
    Substituting this length into Eq.~\eqref{eqn:dpc} gives the result.
    $\qed$
\end{proof}

    
    \section{Classification of Materials Data}\label{sect:classification Methodology}

Here we describe the $\dpc$-distance based classification of crystal structures of
high-entropy alloys (HEAs)
using atom probe tomography (APT) experiments.
    Recall that the building blocks of HEAs are either body-centered cubic (BCC) or face-centerd cubic (FCC).
    Topological considerations are a natural fit for this problem since BCC and FCC crystal structures enjoy
    a different atomic configuration within a unit cell.
    Indeed, the BCC structure has one atom at its center, but the FCC contains a 
    void (recall~\cref{fig:BCC,fig:FCC}). This distinction is 
    important from the viewpoint of persistent homology.

\begin{figure}
    \centering
    \includegraphics[width=\textwidth]{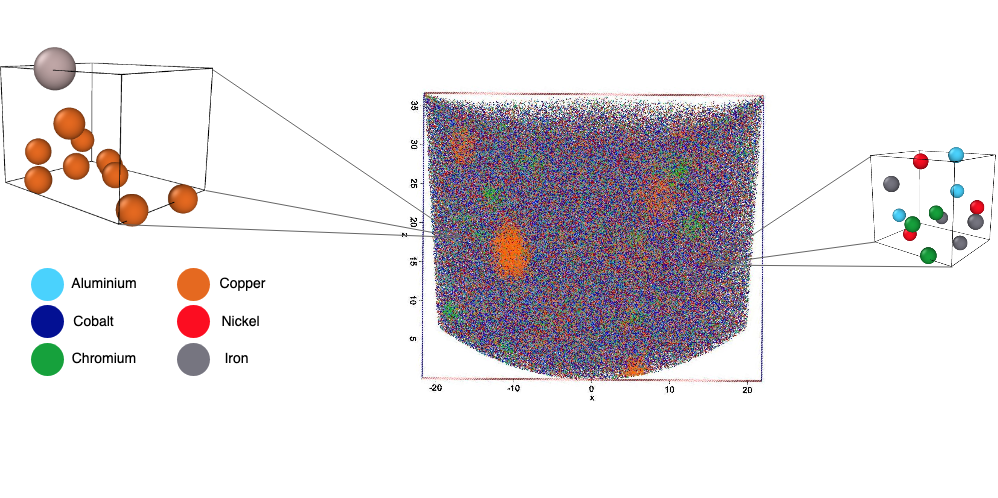}
    \caption{Image of APT data with atomic neighborhoods shown
            in detail on the left and right. Each
            pixel represents a different atom, the neighborhood of which is considered. Certain patterns with distinct crystal structures exist, e.g., the orange region
            is copper-rich (left), but overall no pattern is identified.
            Putting a single atomic neighborhood under a
            microscope, the true crystal structure of
            the material, which could be either
            BCC (\cref{fig:BCC}) or FCC (\cref{fig:FCC}), is not revealed.
            This distinction is obscured due to experimental noise.}\label{fig:APT}
\end{figure}

However, topology alone is insufficient to distinguish between noisy and sparse BCC and FCC lattice
structures accurately. If we count the number of atoms in a unit cell 
(see~\cref{fig:BCC,fig:FCC}) one may see that a BCC unit 
cell has two atoms, one at the center and
    $1/8^{th}$ of an atom at the unit cell's corners, as it shares part of these corner
    atoms with its neighboring cells. Similarly, an FCC unit cell has four atoms;
    the same $1/8^{th}$ of the corner atoms plus one-half of each of the six 
    atoms on the cell's faces. In both cases, the atoms on the faces and lattice points are
    shared with the cell's neighbors and are only counted as a proportion contributing 
    to the unit cell.

Another way to see this difference in cardinality is by plotting the number of connected components against 
the number of holes for both BCC and FCC crystal structures.
\cref{fig:Jitter_plots_a,fig:Jitter_plots_b} depict that FCC structures have larger point clouds, and consequently,
a greater number of connected components. Observe in~\cref{fig:bcc_fcc_pd} 
that the number of connected components and 1-dim holes are greater
in the FCC diagrams than the BCC diagrams.
Consequently, we must account for more than just homological differences
    when considering persistence diagrams derived from these atomic neighborhoods. Variability in the size of the underlying point clouds must be considered, as verified in Proposition~\ref{proposition:wls}. 
    Given the salient topological and
    cardinality differences between these two crystal structures, we seek to
    classify their associated persistence diagrams via these essential differences. To that end, we consider the $d_p^c$ distance given in Eq.~\eqref{eqn:dpc}.

In the numerical experiments, the point clouds (atomic neighborhoods) are extracted from
        a sample containing approximately 10,000 atoms. We remove
        atoms, to create spasity, and add Gaussian noise to the larger sample mirroring
        those levels found in true APT experimental data. To 
        create these neighborhoods, we consider a fixed volume around
        each atom in the perturbed sample and those atoms within 
        the volume are recorded for our classification methodology. Here we consider
    $N=1,000$ synthetic atomic neighborhoods ($N_{BCC} = 500$ BCC structures and $N_{FCC} = 500$ FCC structures) with noise and sparsity levels similar to those found in true APT experiments. Let $\bm{q} = (q_1,\dots, q_M)^T$ be the atoms' positions within an atomic neighborhood. Applying the persistent homology machinery of 
    Section~\ref{sect:ph}, one
    obtains the associated persistence diagram denoted by $X_q$, 
    see~\cref{fig:bcc_fcc_pd}.
For our classification problem, we are interested in the conditional probability,
$\widetilde{\pi}_j = \mathbb{P}(Y_i=j\mid X_i)$,
of the persistence diagram $X_i$ being in class $Y_j$, 
for $j=0$ (BCC) or $j=1$ (FCC).
To that end, we consider a logistic regression model,
\begin{equation}
\log\left(\frac{\widetilde{\pi}_j}{1-\widetilde{\pi}_j}\right) = \alpha + \sum_{i=1}^L\,\varphi_i(\bm{\Sigma}_i), \label{eq:additive_regression}
\end{equation}
where $\varphi_i$ is some pertinent smooth function, and $\bm{\Sigma}\in\R^{N\times 8}$ is the feature matrix whose $i^{th}$ row is
\begin{equation}
    \bm{\Sigma}_i = (\E_{i,B}^{0}, \E_{i,B}^{1}, \Var_{i,B}^{0}, \Var_{i,B}^{1}, 
    \E_{i,F}^{0}, \E_{i,F}^{1}, \Var_{i,F}^{0}, \Var_{i,F}^{1}).\label{eqn:feature_matrix}
    \end{equation}
    For any persistence diagram $X^{k}_i$ with $k$-dimensional homology $(k=0, 1), \\ \E_{i,B}^{k} =  \frac{1}{N_{BCC} } \sum_{j=1}^{N_{BCC}} \dpc(X_i^{k},X_j^{k})$ and $\Var_{i, B}^{k} = \frac1{N_{BCC}-1}\sum_{j=1}^{N_{BCC}}\,(\dpc(X_i^{k},X_j^{k}) - \E_{i,B}^{k})^2$
    respectively yield the average and variance of the distance between $X^{k}_i$ and the collection
    of all BCC persistence diagrams. Similarly, $\E_{i, F}^{k}$ and $\Var_{i, F}^{k}$ are the average and variance of the distance between $X^{k}_i$ and the collection
    of all FCC persistence diagrams.

\begin{figure}
    \centering
    \begin{subfigure}[t]{0.48\textwidth}
        \centering
        \includegraphics[scale=0.3]{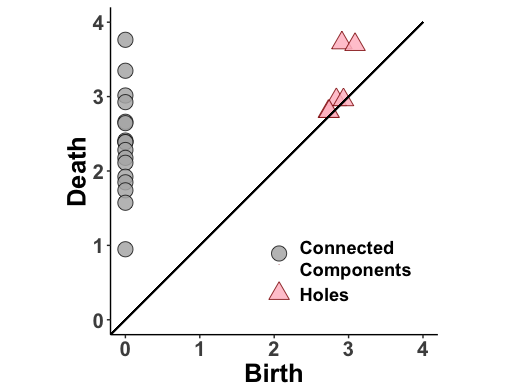}
        \caption{}\label{fig:bcc_pd}
    \end{subfigure}
    \hfill%
    \begin{subfigure}[t]{0.48\textwidth}
        \centering
        \includegraphics[scale=0.3]{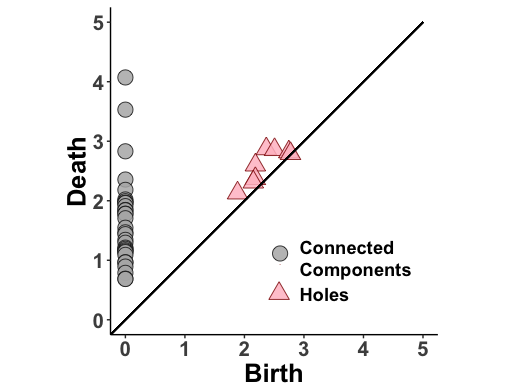}
        \caption{}\label{fig:fcc_pd}
    \end{subfigure}
    \caption{Example of persistence diagrams generated by (\textsc{a}) a 
        BCC lattice, and (\textsc{b}) FCC lattice.
        The data has a noise standard deviation of $\tau=0.75$ and 67\% of the atoms are missing. Note that the BCC diagram has two 
        prominent (far from the diagonal) points representing 1-dim holes and fewer connected components and 1-dim holes than does the FCC diagram.}
        \label{fig:bcc_fcc_pd}
\end{figure}    

We perform 10-fold cross validation on the 1,000 synthetic crystal structures. In other words, the data is
divided randomly into 10 folds, and 9 folds of the data are used as a training set. For any unknown crystal
structure in the remaining fold, the feature vector of the unknown crystal structure is computed according to
Eq.~\eqref{eqn:feature_matrix} and used as input for the decision tree classifier. 
Similarly, the other 9 folds are each used once as test sets employing the same procedure.
The tree
finds the best fit for the features from the additive model in Eq.~\eqref{eq:additive_regression}
and returns the class of the unknown structure. 

\begin{figure}[]
    \centering
    \begin{subfigure}[t]{0.49\textwidth}
        \centering
        \includegraphics[scale=0.28]{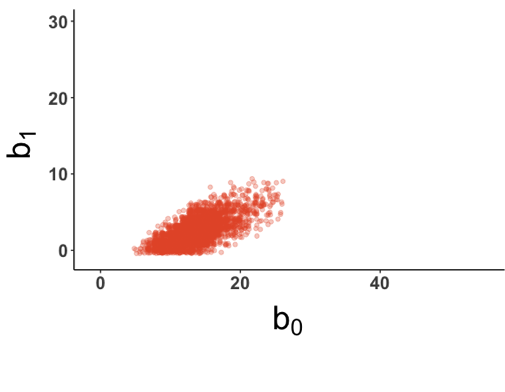}
        \caption{}\label{fig:CI_plots_a}
    \end{subfigure}
    \hfill%
    \begin{subfigure}[t]{0.49\textwidth}
        \centering
        \includegraphics[scale=0.28]{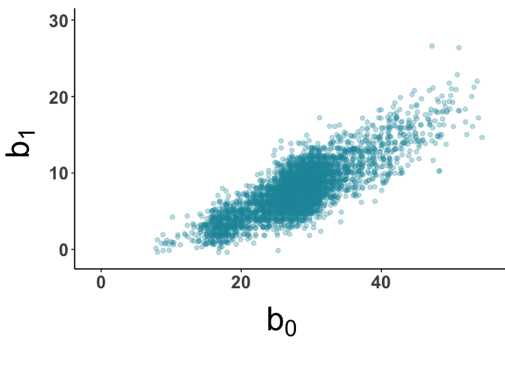}
        \caption{}\label{fig:CI_plots_b}
    \end{subfigure}
    
    \begin{subfigure}[t]{0.49\textwidth}
        \centering
        \includegraphics[scale=0.28]{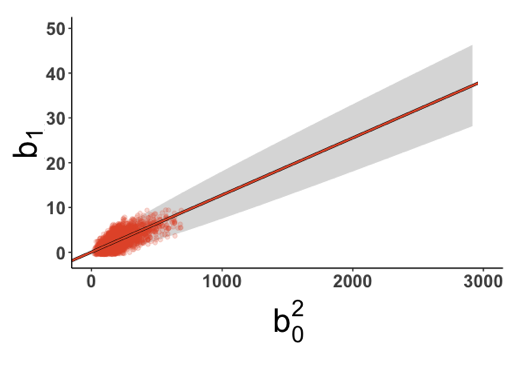}
        \caption{}
        \label{fig:Jitter_plots_a}
    \end{subfigure}
    \hfill%
    \begin{subfigure}[t]{0.49\textwidth}
        \centering
        \includegraphics[scale=0.28]{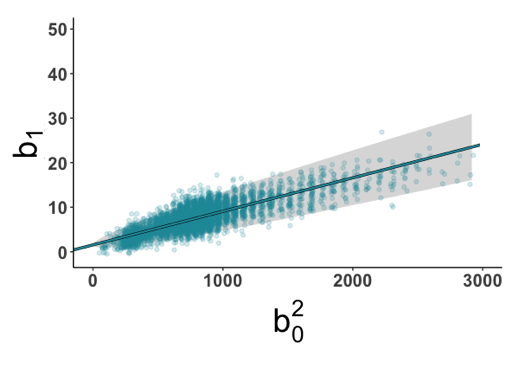}
        \caption{}
        \label{fig:Jitter_plots_b}
    \end{subfigure}
    
    \begin{subfigure}[t]{\textwidth}
        \centering
        \includegraphics[scale=0.1]{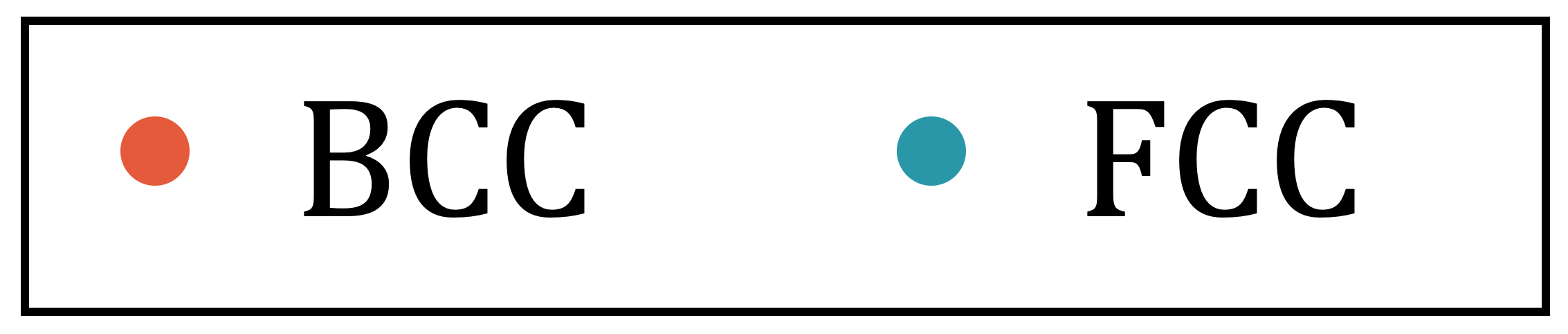}
    \end{subfigure}
    
    \caption{\emph{Top:} Number of connected components (in this case atoms), $\mathbf{b_0}$, against  the
        number of 1-dim homological features, $\mathbf{b_1}$, of the persistence diagrams. 
        One can see the presence of heteroscedasticity since the variance of $\mathbf{b_1}$ increases as
        $\mathbf{b_0}$ increases. 
        \emph{Bottom:} Same as in top but using a quadratic transformation of the predictor variable, along
        with the weighted least squares fit line and 95\% prediction intervals provided by Proposition
        \ref{proposition:wls}.}
\end{figure}     

\begin{table}[]
    \centering    
    \begin{tabular}{*3l}    \toprule
        $\tau$ & $c$-value & \emph{Accuracy} \\\midrule
        0.0 & 0.01 & 99\% \\
        0.25 & 0.05 & 99.4\% \\
        0.75 & 0.03 & 96.5\% \\
        1.0 & 0.13 & 96.4\% \\ \bottomrule
        \hline
    \end{tabular}
    \caption{The atomic positions in the APT data is $\mathcal{N}(0, \tau^2)$ distributed with 67\% of the atoms missing. We employ the $\dpc$ classifier, where $c$ has been optimized in each noise level case. The accuracy in the 10-fold cross validation is listed in the third column.}\label{tab:c_vals}
\end{table}

\begin{figure}
    \centering
    \includegraphics[scale=0.75]{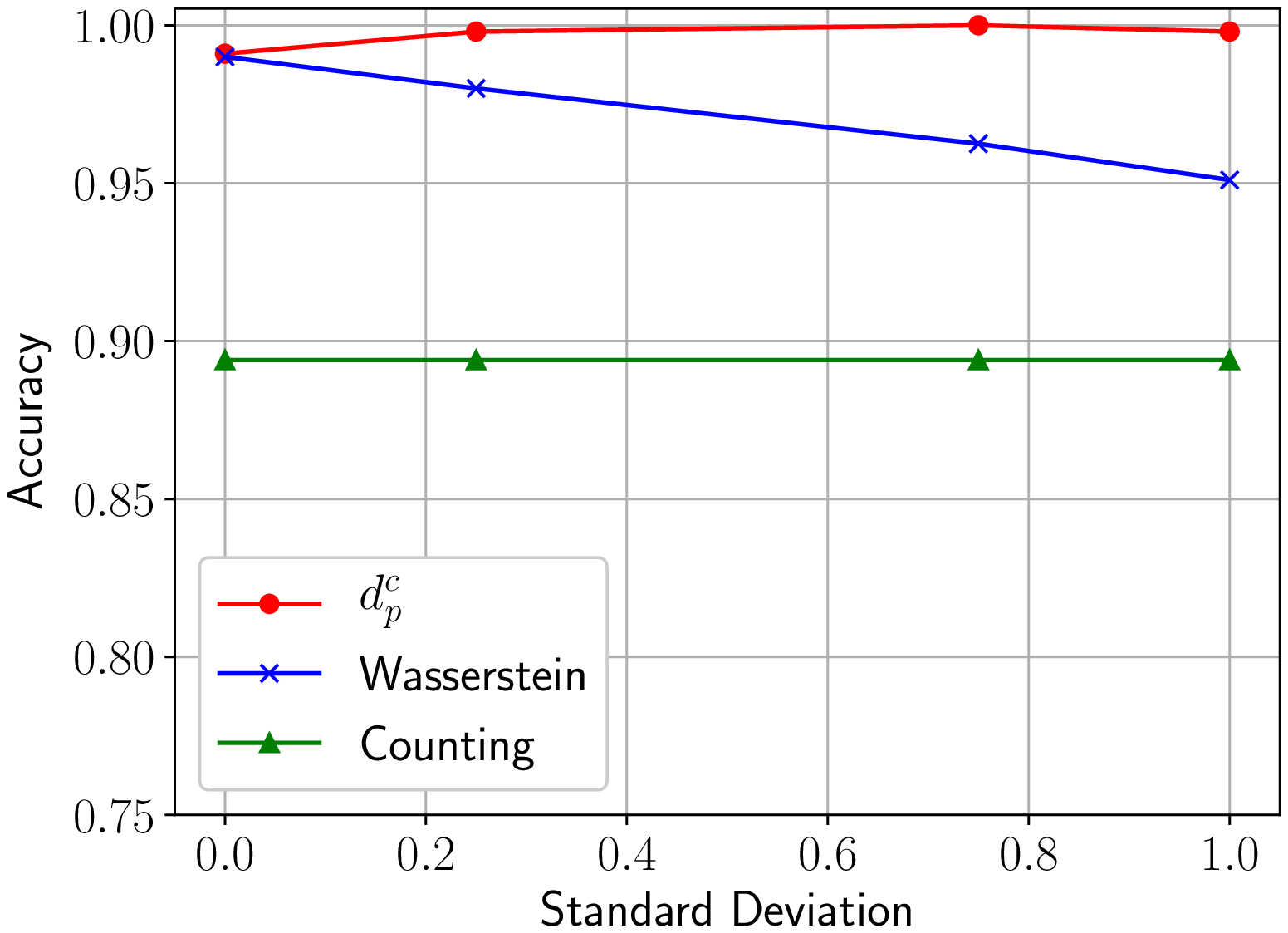}
    \caption{10-fold cross validation accuracy scores for $d_p^c$ (red), Wasserstein (blue), and counting (green) classifiers, plotted against
        different standard deviations, $\tau$, (see Table \ref{tab:c_vals}) of the normally distributed noise of the atomic positions. In each
        instance, the sparsity has been fixed at 67\% of the atoms missing, as in a true APT experiment.}
    \label{fig:cross_val}
\end{figure}

For our numerical experiments, 
the persistence diagrams are constructed using the C++ Ripser 
software, and the scikit-learn decision tree implementation.  The 
studies~\cite{miller2012future, santodonato2015deviation} estimate 
that approximately 65\% of the data is missing.
However, an estimate of the experimental noise is not provided. 
    In fact, as noted by~\cite{LarsonDavidJ2013LEAP,Miller2014APT},
    the noise varies between experiments and specimens.
    Our synthetic data replicates this resolution
    by drawing from a 
    Gaussian~\cite{gault2010spatial,mcnutt2017interfacial,moody2011lattice}, $\mathcal{N}(0, \tau^2)$, with four different levels of variance to give a 
    more representative approximation of true APT datasets. Computing the $d_p^c$ distances with $p=2$ to imitate typical Euclidean distance, we find different values
    of $c$ via a grid search for these four different levels of variance, $\tau^2,$ in both 0- and 1-dim homology, employing a different dataset than is used
        for the classification. In each case, 
    a geometric sequence of 10 values between $0.01$ and $1$
    is taken into account.
    The results and the associated
    algorithmic accuracy are presented in Table~\ref{tab:c_vals}.

As a comparison the feature matrix
    in Eq.~\eqref{eqn:feature_matrix} is also calculated using 
    the Wasserstein distance, choosing $p=2$. Moreover, we adopt a counting classifier which takes into account only the number of points in an atomic neighborhood as the input feature in the tree classifier.
    Our $\dpc$ classifier successfully dichotomizes these 1,000 persistence diagrams generated by 
    BCC and FCC lattice structures
    at better than 96\% accuracy, where accuracy is measured as (1 - Misclassification rate). The $\dpc$ classifier outperforms both the Wasserstein and the counting classifier, see~\cref{fig:cross_val}. These results demonstrate that using just the differences in cardinality 
    between the two classes of crystal structures is insufficient to distinguish between
    them.

As demonstrated in Proposition~\ref{proposition:wls}, there is a relationship between the number of connected
components, $\mathbf{b_0}$, (number of atoms in this case) and the number of 1-dim homological features,
$\mathbf{b_1}$, in the persistence 
diagrams~\cref{fig:CI_plots_a,fig:CI_plots_b} demonstrate
this relationship, as well as the presence of heteroscedasticity between $\mathbf{b_0}$ and $\mathbf{b_1}$,
also verified by the Breusch-Pagan test~\cite{breusch1979simple} with a $p-$value of $9.3\times 10^{-54}$
for FCC cells and a $p-$value of $2.01\times 10^{-47}$ for BCC cells.
    \cref{fig:CI_plots_a,fig:CI_plots_b} also provide 95\% prediction intervals for $\mathbf{b_1}$
    based on the weighted least squares regression analysis of
    Proposition~\ref{proposition:wls}. To that end, this exact fine balance between the number of atoms in a
    neighborhood and the associated topology created by the positions of these atoms in the cubic cell is
    captured by the $\dpc$ distance.

    \section{Conclusions}\label{sect:conclusions}

This work combined statistical learning and topology to
classify the crystal structure of high entropy alloys using atom
probe tomography (APT) experiments. These APT experiments produce a noisy and sparse dataset,
from which we extract
atomic neighborhoods, i.e., atoms within a fixed volume forming a point cloud, and apply the machinery of Topological Data Analysis (TDA) to these point clouds. 
Viewed through the lens of TDA, these point clouds are a rich source of topological information. Indeed,
employing persistent homology, we summarized the shape of these atomic neighborhoods and classified their
crystal
structures as either BCC or FCC. The classifier was based on features derived from the new distance
on persistence diagrams, denoted herein by $d_p^c$. This distance is
different from all other existing distances on
persistence diagrams in that it explicitly penalizes
differences in cardinality between diagrams. 

We proved a stability result for the $d_p^c$ distance, demonstrating that 
small perturbations of the underlying point clouds resulted in small changes to the $d_p^c$
distance. We also provided guidance for the choice of the $c$ parameter by looking at 
confidence bounds using a function of the cardinalities of the persistence diagrams. 

The classification results presented herein could aid materials science researchers by providing
a previously unavailable representation of the local atomic 
environment of high entropy alloys from APT data. The methodology need not be limited to a binary choice
between BCC and FCC, e.g., entropy-stabilized oxides \cite{rost2015entropy} are amenable to APT
characterizations and our process could be generalized to those materials as well. Moreover, as APT
experiments produce datasets on the order
of 10 million atoms, materials science research has moved into 
the realm of big data, and the necessary computational and
modelling tools have yet to be developed for this regime according to \cite{katsoulakis2017special}. 
The $\dpc$ classifier, coupled with our ongoing research of quantifying local atomic distributions as
in~\cite{spannaus}, aims to recover global atomic structure of high entropy alloys.

\section*{Acknowledgments}

The authors would like to thank the anonymous associate editor and two anonymous reviewers for their insightful
comments which substantially improved the manuscript. Moreover,
the authors would like to thank Professor David J.~Keffer (Department of Materials Science and Engineering at The University of Tennessee) for providing 
the codes which create the realistic APT datasets and for useful 
discussions, as well as Professor Kody J.H. Law (School of Mathematics
at the University of Manchester)  for
insightful discussions.

\bibliographystyle{acm}
\bibliography{tex/references.bib}

\end{document}